\documentclass[12pt]{article}

\usepackage{amsmath, amssymb, amsthm, bm, graphicx, booktabs, natbib, geometry, microtype}
\usepackage{enumitem}
\usepackage{tabularx}
\usepackage{array}
\usepackage{algorithm}
\usepackage{algorithmic}
\usepackage{placeins}

\numberwithin{equation}{section}

\usepackage[switch]{lineno}

\theoremstyle{plain}
\newtheorem{theorem}{Theorem}[section]

\newtheorem{corollary}[theorem]{Corollary}

\theoremstyle{definition}

\newtheorem{assumption}[theorem]{Assumption}

\theoremstyle{remark}
\newtheorem{remark}[theorem]{Remark}

\usepackage[hidelinks]{hyperref}

\title{Ribbon: Scalable Approximation and Robust Uncertainty Quantification}
\author{Graham Gibson\thanks{\texttt{gcgibson@lanl.gov}}
\and John Tipton
\and Kellin Rumsey
\and Natalie Klein}
\date{}

\begin{document}
\maketitle

\begin{abstract}
Reliably quantifying predictive uncertainty is difficult for complex, high-dimensional, or misspecified models. Both fully Bayesian and bootstrap resampling methods provide principled uncertainty estimates but are often too expensive for modern machine-learning models because they require posterior sampling or repeated model refitting. We introduce \emph{Ribbon}, a scalable approximation to Dirichlet-reweighted bootstrap uncertainty. Ribbon replaces repeated refitting with an influence-function linearization around a single fitted model, preserving the first-order data-reweighting structure of the Bayesian bootstrap while requiring only post-hoc linear algebra. Ribbon approximates the Bayesian-bootstrap or weighted-likelihood-bootstrap refitting target. With a general concentration parameter $\alpha$, Ribbon gives a calibrated Dirichlet-reweighting family whose uncertainty scale can be tuned on validation data. We show that Ribbon is asymptotically equivalent to a flat-prior Laplace approximation under correct likelihood specification and recovers the robust sandwich covariance under misspecification. Across synthetic regression, MNIST classification, and California Housing benchmarks, Ribbon provides competitive predictive performance and improved calibration in several settings while avoiding repeated model retraining.
\end{abstract}

\section{Introduction}

Reliably quantifying predictive uncertainty is a central challenge in modern statistical
and machine--learning models, particularly when those models are high--dimensional,
complex, or misspecified. In many scientific and engineering applications---such as climate
modeling, epidemiology, and autonomous systems---accurate point predictions alone are
insufficient: models must also provide calibrated measures of uncertainty
that reflect both data variability and limitations of the modeling assumptions.

Fully Bayesian inference offers a coherent probabilistic framework through posterior distributions over parameters and predictions \citep{gelman2013bayesian}. However, exact Bayesian computation is rarely tractable for modern models such as deep neural networks. Markov chain Monte Carlo methods can, in principle, be applied, but they scale poorly with model size and often suffer from slow mixing, particularly in high-dimensional or multimodal posterior landscapes \citep{izmailov2021what,papamarkou2022challenges}.

As a result, much recent work has focused on scalable approximate Bayesian methods.
A widely used class of such approaches is based on local Gaussian approximations to the
posterior, including the Laplace approximation \citep{daxberger2021laplace}
and related variational inference methods \citep{blundell2015weight,gal2016dropout}.
While Laplace approximates posterior uncertainty using curvature and variational inference fits an approximating distribution using a divergence objective, both methods yield tractable or easily sampled distributions over parameters.

Despite their scalability, standard Laplace or variational approximations exhibit well--known limitations. Standard curvature-only Laplace methods, while asymptotically correct under ideal likelihood specification, rely on a local quadratic approximation to the loss and use local curvature to quantify posterior concentration. Under model misspecification, curvature-only uncertainty estimates need not match the sampling uncertainty of the fitted estimator, even when point predictions are accurate.
Variational inference methods can underestimate posterior uncertainty due to restrictive variational families and optimization bias \citep{yao2018yes,wilson2020bayesian}. 

Resampling--based procedures provide an appealing alternative. Classical bootstrap methods \citep{efron1979bootstrap} and the Bayesian bootstrap \citep{rubin1981bayesian} approximate uncertainty by resampling or repeatedly reweighting the empirical data distribution and refitting the model. These approaches are model--agnostic in the sense that they perturb the empirical distribution rather than requiring a specific parametric posterior, and often yield useful predictive intervals under mild model misspecification. Their robustness stems from the fact that uncertainty is driven by perturbations of the data rather than local curvature assumptions alone. However, this robustness comes at a substantial computational cost: each resample typically requires retraining or reoptimizing the model, which is prohibitive for modern large--scale models.

This tension between scalable but sometimes brittle curvature--based approximations and robust but expensive resampling methods motivates a unified approach that combines the strengths of both. A key tool in this direction is the influence function, a classical concept from robust statistics that characterizes the infinitesimal effect of perturbing a single observation on an estimator \citep{hampel1974influence,huber2009robust}. Influence functions underlie familiar asymptotic variance estimators, including the sandwich covariance estimator \citep{white1980heteroskedasticity}, and have recently been revisited in machine learning for tasks such as interpretability, data valuation, and robustness analysis \citep{koh2017understanding,pruthi2020estimating}.

Recent work has shown that influence--function linearizations can also be used to approximate uncertainty induced by resampling procedures. In particular, linearizing the effect of weighted empirical risk minimization reveals a close connection between influence functions, the Bayesian bootstrap, and local Gaussian posterior approximations \citep{newton1994approximate,arora2021rue}. Under correct likelihood specification, these linearized resampling distributions coincide asymptotically with the flat-prior Laplace approximation. Under misspecification, however, they naturally generalize to the robust sandwich covariance form, capturing sampling variability ignored by standard curvature-only approximations.

In this work, we build on these insights to introduce Ribbon, a scalable framework for predictive uncertainty quantification within the class of differentiable estimators trained by smooth empirical risk minimization. Ribbon approximates the standard Dirichlet-reweighted Bayesian-bootstrap refitting procedure when $\alpha=1$ by propagating Dirichlet--weighted perturbations of the empirical data distribution through an influence--function linearization around a single fitted model. For general fixed $\alpha>0$, Ribbon should be understood as a frequentist--calibrated Dirichlet-reweighting generalization rather than the canonical Bayesian bootstrap itself.

Crucially, Ribbon partially separates directional uncertainty from global calibration. Directional uncertainty is determined by the influence covariance, involving the curvature and the empirical gradient second moment, while the global scale of uncertainty is governed by a single Dirichlet concentration parameter. This parameter can be tuned using empirical calibration criteria, providing a transparent mechanism for aligning predictive uncertainty with observed frequencies. Such tuning can correct global under- or over-dispersion, although it cannot repair arbitrary directional errors caused by poor curvature or gradient-covariance approximations.

We show that Ribbon recovers the flat-prior Laplace approximation asymptotically under correct likelihood specification when the standard Bayesian-bootstrap concentration $\alpha=1$ is used, while automatically expanding to the robust sandwich form under misspecification. For general fixed concentration values, Ribbon yields a temperature-scaled version of this covariance. Empirically, across regression and classification benchmarks---including neural networks with damped and structured curvature approximations---Ribbon produces competitive and often better calibrated uncertainty estimates than common scalable alternatives.

Our contributions are threefold. First, we formalize the Ribbon framework and establish its theoretical connection to the Bayesian bootstrap, influence functions, and Laplace approximations under both correct specification and misspecification. Second, we provide asymptotic guarantees demonstrating first--order covariance accuracy and validation-based calibration consistency under explicit regularity, monotonicity, and uniform convergence assumptions. Third, we demonstrate empirically that Ribbon delivers robust, well--calibrated uncertainty estimates for several machine--learning models without the computational burden of repeated retraining.

The remainder of the paper proceeds as follows.
Section~\ref{sec:methods} introduces the Ribbon construction and its connection to Dirichlet-reweighted bootstrap uncertainty.
Section~\ref{sec:properties} discusses theoretical properties of the Ribbon estimator, including its relationship to Laplace and sandwich covariance approximations.
Section~\ref{sec:baselines} describes the baseline uncertainty-quantification methods and implementation details.
Section~\ref{sec:experiments} evaluates Ribbon empirically on regression and classification benchmarks.
Section~\ref{sec:discussion} concludes with limitations and future directions.
Appendix~\ref{sec:theory} provides formal asymptotic results and proofs.

\section{Methods}
\label{sec:methods}

\begin{table}[t]
\centering
\caption{Summary of notation used throughout the paper.}
\label{tab:notation}
\renewcommand{\arraystretch}{1.2}
\begin{tabularx}{\textwidth}{>{\raggedright\arraybackslash}p{5.25cm}
                                  >{\raggedright\arraybackslash}X
                                  >{\raggedright\arraybackslash}p{4cm}}
\toprule
\textbf{Symbol} & \textbf{Definition} & \textbf{Role / Interpretation} \\
\midrule
$z_i = (x_i, y_i)$ & Training observation & Data \\
$\ell(z_i,\theta)$ & Per-sample loss function & Objective \\
$L_n(\theta) = n^{-1}\sum_{i=1}^n \ell(z_i,\theta)$ & Empirical risk & Objective \\
$r(\theta)$ & Optional fixed regularization penalty on the same scale as $L_n$ & Penalization \\
$\hat\theta$ & Unweighted ERM or penalized ERM & Point estimate \\
$w = (w_1,\dots,w_n)$ & Dirichlet weights, $w \sim \mathrm{Dirichlet}(\alpha, \ldots, \alpha)$ & Resampling \\
$\tilde w = n w - j$ & Centered count-scale weight perturbation & Resampling \\
$\hat\theta_w$ & Weighted ERM minimizer & Target parameter \\
$\tilde\theta_w$ & Influence-function approximation to $\hat\theta_w$ & Approximation \\
$g_i = \nabla_\theta \ell(z_i,\hat\theta)$ & Per-sample empirical-loss gradient at $\hat\theta$ & Sensitivity \\
$G = [g_1,\dots,g_n]^\top$ & Stacked gradient matrix & Sensitivity \\
$H_i = \nabla^2_\theta \ell(z_i,\hat\theta)$ & Per-sample Hessian & Curvature \\
$H = n^{-1}\sum_{i=1}^n H_i$ & Mean empirical-loss Hessian at $\hat\theta$ & Curvature \\
$H_r = H + \nabla^2 r(\hat\theta)$ & Penalized curvature, when a fixed penalty is used & Curvature \\
$H_F = n^{-1}\sum_{i=1}^n g_i g_i^\top$ & Empirical gradient second moment; empirical Fisher-type quantity for likelihood losses & Score variance \\
$H_{F,c}=n^{-1}G^\top(I-n^{-1}jj^\top)G$ & Centered gradient second moment & Score variance \\
$U_n(w) = \sum_{i=1}^n (w_i-1/n) g_i$ & Centered weighted score perturbation & Random direction \\
$\Delta_{\mathrm{IF}} = -H^{-1} U_n(w)$ & Influence-function parameter update, unregularized exact-ERM case & Approximation \\
$f(x;\theta)$ & Predictive model & Prediction \\
$J_x = \partial f(x;\hat\theta)/\partial\theta^\top$ & Prediction Jacobian & Pushforward \\
$\alpha$ & Dirichlet concentration parameter & Uncertainty scaling \\
\bottomrule
\end{tabularx}
\end{table}

\subsection{Overview}

The central idea of Ribbon is to approximate Dirichlet-reweighted bootstrap uncertainty by propagating perturbations of the empirical measure through an influence--function linearization at a single fitted estimator. Let $\mathcal{D}_n = \{z_i = (x_i, y_i)\}_{i=1}^n$ denote the training data, and let $w = (w_1, \ldots, w_n)$ be a vector of random weights drawn from a symmetric $\mathrm{Dirichlet}(\alpha, \dots, \alpha)$ distribution.
The standard Bayesian bootstrap corresponds to $\alpha=1$. Values $\alpha\ne1$ produce a temperature-scaled Dirichlet reweighting family used for frequentist calibration.

We consider the case where we estimate a predictive model $f(x; \theta)$ using loss function $\ell$.
The weighted ERM is defined as
\begin{equation}
  \hat{\theta}_w
  = \arg\min_{\theta} \sum_{i=1}^n w_i\, \ell(z_i, \theta),
  \label{eq:wlb}
\end{equation}
where $\hat{\theta}_w$ denotes the fitted parameters under weighting scheme $w$.
For $\alpha=1$, Eq.~\eqref{eq:wlb} is the Dirichlet-reweighted refitting target associated with the Bayesian bootstrap or weighted likelihood bootstrap. Ribbon approximates this weighted empirical risk minimizer without repeated refitting. The approximation uses a first--order influence update around the unweighted ERM $\hat\theta$, the minimizer of the unweighted empirical loss obtained when $w_i = 1/n$ for all $i$.

Define the per--example gradient and Hessian at the unweighted ERM $\hat{\theta}$ as
\[
  g_i = \nabla_\theta \ell(z_i, \hat{\theta}),
  \qquad
  H_i = \nabla^2_\theta \ell(z_i, \hat{\theta}),
\]
and let $H = n^{-1} \sum_{i=1}^n H_i$ denote the average curvature matrix,
$G = [g_1, \ldots, g_n]^{\!\top}$ the stacked gradient matrix, and
$j=(1,\ldots,1)^\top \in \mathbb{R}^n$ the length-$n$ vector of ones.

We will use two empirical second-moment matrices throughout. The first is the average curvature
\[
  H = \frac{1}{n}\sum_{i=1}^n H_i,
\]
which describes the local sensitivity of the fitted objective to parameter perturbations. The second is the empirical gradient second moment
\[
  H_F = \frac{1}{n}\sum_{i=1}^n g_i g_i^\top
      = \frac{1}{n}G^\top G,
\]
which describes the variability of the per-observation score or loss-gradient contributions. When the empirical gradients do not sum exactly to zero, as can occur with regularization, approximate optimization, or numerical error, we instead use the centered gradient second moment
\[
  H_{F,c}
  =
  \frac{1}{n}G^\top
  \left(I-\frac{1}{n}jj^\top\right)G .
\]
Thus $H$ controls local curvature, while $H_F$ or $H_{F,c}$ controls the directions in which data reweighting perturbs the estimator.

In classical robust statistics, the influence function of an estimator describes the first-order effect of an infinitesimal perturbation of the data distribution on the estimated parameter. In the present finite-sample empirical-risk setting, the relevant per-observation influence direction at $\hat\theta$ is
\[
  \operatorname{IF}_i
  =
  -H^{-1}g_i .
\]
Thus $\operatorname{IF}_i$ is the local parameter displacement associated with upweighting observation $z_i$, under the normalization induced by the averaged empirical loss. Ribbon combines these per-observation influence directions using centered Dirichlet weight perturbations.

The weighted estimator satisfies the first-order optimality condition
\[
\sum_{i=1}^n w_i \nabla_\theta \ell(z_i,\hat\theta_w) = 0.
\]
Expanding each gradient term around the unweighted ERM $\hat\theta$ gives
\[
\nabla_\theta \ell(z_i,\hat\theta_w)
\approx
 g_i + H_i(\hat\theta_w - \hat\theta).
\]
Substitution yields
\[
0 \approx \sum_{i=1}^n w_i g_i
\;+\
\left(\sum_{i=1}^n w_i H_i\right)(\hat\theta_w - \hat\theta).
\]
At an equally-weighted ERM, $\sum_i g_i=0$, and therefore
\[
\sum_{i=1}^n w_i g_i
=
\sum_{i=1}^n (w_i-1/n)g_i
=
\frac{1}{n}G^\top\tilde w,
\qquad
\tilde w = n w - j,
\]
where $j$ is the length-$n$ vector of ones defined above.
Replacing the weighted Hessian by the mean Hessian $H$ gives the influence approximation
\begin{equation}
  \Delta\theta_w
  \approx
  -H^{-1}\!\left(\frac{1}{n} G^{\!\top} \tilde{w}\right)
  =
  \frac{1}{n}\sum_{i=1}^n \tilde w_i\operatorname{IF}_i,
  \label{eq:influence}
\end{equation}
where $\Delta\theta_w = \hat\theta_w-\hat\theta$.

Equation~\eqref{eq:influence} gives a first--order approximation to the parameter change induced by Dirichlet reweighting. Each observation contributes its finite-sample influence direction $\operatorname{IF}_i$, scaled by the deviation of its weight from uniform weighting. All computation therefore reduces to evaluating products of $H^{-1}$ with vectors, which can be performed efficiently using iterative solvers or curvature approximations without explicitly forming or inverting $H$.

If the fitted estimator includes a fixed regularization penalty $r(\theta)$, we define the training objective as $L_n(\theta)+r(\theta)$, where $r$ is on the same scale as the averaged empirical loss and is not reweighted by bootstrap weights. The perturbation acts only on the empirical loss terms while the penalty remains fixed across bootstrap draws. The penalized first-order condition is
\[
0=
\sum_{i=1}^n w_i \nabla\ell(z_i,\hat\theta_w)+\nabla r(\hat\theta_w).
\]
Expanding around the unweighted penalized solution gives
\begin{equation}
\Delta\theta_w
\approx
-
\left(H+\nabla^2r(\hat\theta)\right)^{-1}
\sum_{i=1}^n (w_i-1/n)g_i.
\label{eq:penalized-if}
\end{equation}
Thus the centered perturbation $\sum_i(w_i-1/n)g_i$ is the appropriate driving term in both unregularized and penalized settings. In the unregularized exact-ERM case, this reduces to Eq.~\eqref{eq:influence}.
In penalized, approximately optimized, or numerically inexact settings, the centered gradient second moment
\[
H_{F,c}
=
\frac{1}{n}G^\top\left(I-\frac{1}{n}jj^\top\right)G
\]
should be used in covariance calculations in place of the uncentered quantity
\[
H_F = \frac{1}{n}G^\top G,
\]
unless $G^\top j=0$ holds to numerical tolerance.

\subsection{Predictive Distribution}

Once the model $f(x;\theta)$ has been trained to obtain $\hat{\theta}$, uncertainty in predictions arises from uncertainty in $\theta$ (epistemic/reducible uncertainty) and uncertainty due to residual error (aleatoric/irreducible uncertainty). Given a collection of influence--function parameter perturbations $\{\Delta\theta_w^{(b)}\}_{b=1}^B$ induced by Dirichlet reweighting, there are two methods to propagate epistemic uncertainty into prediction space.

The first way to predict at a test input $x_\ast$ is to explicitly resample predictions by evaluating the model at perturbed parameters,
\[
  f^{(b)}(x_\ast) = f(x_\ast;\,\hat{\theta} + \Delta\theta_w^{(b)}).
\]
This closely mirrors the predictive mechanism of the Bayesian bootstrap, but it requires $B$ additional forward passes per test input to generate a collection of predictions $\{f^{(b)}(x_\ast)\}_{b=1}^B$ from which empirical quantiles of interest can be calculated.

The second method uses a linearized predictive update. For a test input $x_\ast$, define the Jacobian
\[
  J_{x_\ast} = \frac{\partial f(x_\ast;\hat\theta)}{\partial\theta^\top},
\]
so that $J_{x_\ast}\Delta\theta$ has the same dimension as $f(x_\ast)$.
For the $b$--th bootstrap sample, associated with Dirichlet weights $w^{(b)}$, the linearized predictive update is
\begin{equation}
  f^{(b)}(x_\ast)
  \approx
  f(x_\ast; \hat\theta)
  + J_{x_\ast} \Delta\theta_w^{(b)},
  \qquad
  \Delta\theta_w^{(b)} \approx
  -H^{-1}\!\left(\frac{1}{n}G^{\!\top}\tilde{w}^{(b)}\right),
  \label{eq:predictive}
\end{equation}
where $\tilde{w}^{(b)} = n w^{(b)} - j$. We refer to this second implementation as the \emph{Ribbon pushforward} approximation, because the Dirichlet-induced parameter perturbations are pushed forward into prediction space through the local Jacobian $J_{x_\ast}$. In contrast, the direct perturbed-parameter implementation evaluates the nonlinear model at $\hat\theta+\Delta\theta_w^{(b)}$ for each draw. Unless otherwise stated, ``Ribbon pushforward'' in the experiments denotes the linearized predictive update in Eq.~\eqref{eq:predictive}. In penalized settings, $H^{-1}$ is replaced by
$H_r^{-1}=(H+\nabla^2r(\hat\theta))^{-1}$.
When predictive covariances are computed analytically rather than by Monte Carlo draws,
$H_F$ should likewise be replaced by the centered gradient second moment $H_{F,c}$ in penalized, approximately optimized, or numerically inexact settings, unless $G^\top j=0$ holds to numerical tolerance.

In regression problems, epistemic uncertainty is only one component of predictive uncertainty; the other arises from aleatoric uncertainty. In the experiments below, we use a shared homoskedastic residual estimate $\hat\sigma^2$ across all methods to isolate differences in epistemic uncertainty. More generally, Ribbon can be combined with a heteroskedastic residual model $\hat\sigma^2(x)$.
Posterior--predictive draws take the form
\[
  y^{(b)}(x_\ast)
  = f^{(b)}(x_\ast) + \varepsilon^{(b)}(x_\ast),
  \qquad
  \varepsilon^{(b)}(x_\ast) \sim N\!\big(0,\,\hat{\sigma}^2(x_\ast)\big).
\]

For classification tasks (and generalized models), predictive uncertainty is propagated through the latent output $\eta(x;\theta)$ prior to the softmax (link function) transformation:
\[
  \eta^{(b)}(x_\ast)
  \approx
  \eta(x_\ast; \hat{\theta})
  + J_{\eta,x_\ast}\Delta\theta_w^{(b)},
  \qquad
  \phi^{(b)}(x_\ast)
  = \mathrm{softmax}\!\big(\eta^{(b)}(x_\ast)\big).
\]
Because the mapping from latent space to probabilities is nonlinear, the distribution of $\phi^{(b)}(x_\ast)$ can deviate from Gaussianity even if the linearized logits are approximately normal.

\subsection{Computation and Algorithmic Procedure}

Ribbon is designed to be computationally efficient once a base model has been trained because uncertainty quantification is performed as a post--hoc operation that reuses gradients and curvature information from the trained model. The method requires only a single optimization run to obtain $\hat{\theta}$; subsequent bootstrap draws involve linear--algebraic operations.

After training, we estimate a local curvature operator $H$ using Hessian-vector products, generalized Gauss--Newton approximations, empirical Fisher approximations, Kronecker-factored approximate curvature estimates that preserve layer-wise structure (KFAC) \citep{martens2015kfac}, or low--rank and diagonal truncations with an associated damping parameter $\lambda$. Each bootstrap draw then requires a linear solve with $H$ or its factorized approximation and one Jacobian--vector product per test batch to propagate the perturbation.

\begin{algorithm}[t]
\caption{Ribbon}
\label{alg:ribbon}
\begin{algorithmic}[1]
\REQUIRE Training data $\{(x_i, y_i)\}_{i=1}^n$, base model $f(x;\theta)$, number of Monte Carlo draws $B$, concentration parameter grid $\mathcal A$
\ENSURE Predictive samples $\{f^{(b)}(x_\ast)\}_{b=1}^B$

\STATE Train a base model by optimizing the standard unweighted empirical risk or penalized empirical risk to obtain $\hat{\theta}$.
\STATE Construct a damped curvature operator $H$ or $H_r$ using a chosen approximation, such as GGN, KFAC, empirical Fisher, or low--rank SVD, with damping parameter $\lambda$.
\STATE Tune the Dirichlet concentration $\alpha$ on a validation split. For $\alpha=1$, the method approximates the standard Bayesian-bootstrap refitting target; for $\alpha\ne1$, it gives a calibrated Dirichlet-reweighting generalization.
\FOR{$\alpha \in \mathcal{A}$}
\FOR{$b = 1,\dots,B$}
  \STATE Draw weights $w^{(b)} \sim \mathrm{Dirichlet}(\alpha, \ldots, \alpha)$ and form $\tilde{w}^{(b)} = n w^{(b)} - j$.
  \STATE Compute the linearized parameter shift
        \[
          \Delta\theta_w^{(b)} \approx -H^{-1}\!\left(\frac{1}{n} G^{\!\top} \tilde{w}^{(b)}\right),
        \]
        or use the penalized version in Eq.~\eqref{eq:penalized-if} when a fixed penalty is present.
  \STATE Form predictive draws $f^{(b)}(x)$ via Eq.~\eqref{eq:predictive} or by evaluating $f(x;\hat\theta+\Delta\theta_w^{(b)})$.
\ENDFOR
\ENDFOR
\STATE Summarize predictive means, quantiles, coverage, RMSE, CRPS, NLL, Brier score, or ECE as appropriate for the task.
\end{algorithmic}
\end{algorithm}

Because each step after training involves matrix--vector products and linear solves, Ribbon is readily parallelizable across bootstrap samples. The post-hoc computational footprint is typically far smaller than repeated retraining and is dominated by matrix--vector solves and Jacobian-vector products.

\section{Properties of Ribbon}
\label{sec:properties}

We next outline the properties of the Ribbon estimator. Formal proofs are provided in Appendix~\ref{sec:theory}. The formal results are stated for regular finite-dimensional M-estimation. Neural-network experiments should be viewed as empirical evidence that the approximation remains useful with damping and structured curvature approximations, rather than as direct consequences of the classical asymptotic theory.

\subsection{Connection to the Laplace Approximation}
\label{sec:laplace}

To connect Ribbon to the classical Laplace approximation, consider first the standard Bayesian bootstrap with $\mathrm{Dirichlet}(1,\ldots,1)$ weights. In the unregularized exact-ERM case,
\begin{equation}
     \operatorname{Cov}(\Delta\theta_w)
  =
  \frac{1}{n+1} H^{-1} H_F H^{-1},
  \label{eq:dirichlet1-cov}
\end{equation}
which is asymptotically $n^{-1}H^{-1}H_FH^{-1}$.

When $\ell$ is the negative log likelihood of a correctly specified regular parametric model, the information identity gives $E(H_F)=E(H)$. Under the corresponding empirical convergence,
\[
  \mathrm{Cov}(\Delta\theta_w) \approx \frac{1}{n}H^{-1}.
\]
This covariance is identical, to first order, to that produced by the classical flat-prior Laplace approximation to the Bayesian posterior,
\[
  p(\theta\mid\mathcal{D}_n)
  \approx
  N\!\big(\hat{\theta}, H^{-1}/n\big).
\]
Hence, Ribbon and the flat-prior Laplace approximation coincide asymptotically under correct likelihood specification and standard Bayesian-bootstrap scaling.

For general fixed concentration $\alpha>0$, the same calculation gives
\begin{equation}
  \operatorname{Cov}(\Delta\theta_w)
  =
  \frac{1}{n\alpha+1}H^{-1}H_FH^{-1}.
  \label{eq:alpha-cov-main}
\end{equation}
Under correct likelihood specification this becomes approximately $(n\alpha+1)^{-1}H^{-1}$, so $\alpha$ acts as a temperature-like scaling of the Laplace covariance. The standard Bayesian bootstrap corresponds to $\alpha=1$.

When the model or loss is misspecified, the curvature of the empirical loss $H$ no longer matches the variability of the gradients, so that $H_F \neq H$.
For $\alpha=1$, the covariance is asymptotically
\begin{equation}
  \mathrm{Cov}(\Delta\theta_w)
  \approx \frac{1}{n}H^{-1}H_FH^{-1},
  \label{eq:sandwich}
\end{equation}
the familiar sandwich or heteroskedasticity--robust covariance estimator \citep{white1980heteroskedasticity}. Under misspecification, the usual curvature-based Laplace approximation reflects local posterior curvature and does not generally match the frequentist sampling covariance of the estimator. The influence-function bootstrap instead recovers the sandwich covariance, which is the robust first-order sampling covariance for M-estimators.

For a test input $x_\ast$ with Jacobian $J_{x_\ast} = \partial f(x_\ast; \hat{\theta})/\partial\theta^\top$, the predictive uncertainty for the linearized predictor follows directly:
\[
  \mathrm{Cov}[f(x_\ast)]
  = J_{x_\ast}\mathrm{Cov}(\Delta\theta_w)J_{x_\ast}^\top.
\]
For $\alpha=1$, this is asymptotically
\[
  \mathrm{Cov}[f(x_\ast)]
  \approx
  \begin{cases}
    J_{x_\ast}H^{-1}J_{x_\ast}^\top/n, & \text{if } H_F \approx H,\\[4pt]
    J_{x_\ast}H^{-1}H_FH^{-1}J_{x_\ast}^\top/n, & \text{otherwise.}
  \end{cases}
\]
In penalized or approximately optimized settings, the same expressions use $H_r$ and $H_{F,c}$ in place of $H$ and $H_F$ where appropriate.

The connection to Laplace provides a clear interpretation: Ribbon produces the same local Gaussian covariance as flat-prior Laplace under correct likelihood specification with $\alpha=1$, but replaces a fixed curvature-only covariance with perturbations of the empirical distribution. Under misspecification, Ribbon recovers a robust first-order covariance structure through the $H_F$ term.

\subsection{Behavior Away from the Training Distribution}

In many smooth models, predictive uncertainty can increase as inputs move into regions where the fitted function is weakly constrained by training data. Ribbon can exhibit this behavior through the sensitivity of predictions to influence-induced parameter perturbations, although this is not guaranteed for arbitrary architectures or distribution shifts.

Ribbon approximates the effect of Dirichlet--weighted resampling through perturbations of the model parameters,
\[
\theta^{(b)} \approx \hat{\theta} + \Delta \theta_w^{(b)}.
\]
Predictive uncertainty at a test input $x_\ast$ depends on the sensitivity of the model output at $x_\ast$ to such perturbations. For inputs well supported by the training data, the model output may be stable under small parameter perturbations. For inputs outside the training distribution, the model output may become more sensitive to changes in the parameters, leading to amplified variability across resampled predictions \citep{novak2018sensitivity}.

This behavior mirrors that of local Gaussian approximations such as the Laplace approximation, in which predictive uncertainty is governed by the interaction between parameter uncertainty and local sensitivity of the model at the test point. Ribbon does not provide formal guarantees of optimal coverage under arbitrary distributional shift. Its behavior away from the training distribution reflects the influence structure of the fitted model and the adequacy of the local approximation.

\subsection{Dirichlet Concentration and Calibration}

Ribbon draws random weights $w = (w_1, \ldots, w_n)$ from a symmetric $\mathrm{Dirichlet}(\alpha, \ldots, \alpha)$ distribution. Because $w_i > 0$ and $\sum_i w_i = 1$, we express the random perturbations on the count scale as
\[
  \tilde{w} = n w - j,
\]
so that $\mathbb{E}[\tilde{w}] = 0$ and $\sum_i \tilde{w}_i = 0$.
For symmetric Dirichlet weights,
\[
\operatorname{Cov}(\tilde w)
=
\frac{n}{n\alpha+1}
\left(I-\frac{1}{n}jj^{\top}\right).
\]
Using the centered perturbation $U_n(w)=\sum_i(w_i-1/n)g_i=n^{-1}G^\top\tilde w$, the conditional covariance is
\[
\operatorname{Cov}\!\left(\frac1nG^\top\tilde w\right)
=
\frac{1}{n\alpha+1}H_F
\]
when $G^\top j=0$. More generally, the same formula holds with $H_F$ replaced by the centered gradient second moment
\[
H_{F,c}=\frac1n G^\top\left(I-\frac1njj^\top\right)G.
\]
Thus the linearized parameter covariance is
\[
\operatorname{Cov}(\Delta\theta_w)
=
\frac{1}{n\alpha+1}H^{-1}H_FH^{-1}
\]
in the unregularized exact-ERM case, or the corresponding centered-gradient and penalized-curvature version when a fixed penalty is present.

The concentration parameter $\alpha$ acts as an inverse dispersion parameter: smaller $\alpha$ produces more variable Dirichlet weights and larger parameter dispersion, while larger $\alpha$ shrinks the weights toward uniformity and reduces uncertainty. In the limiting case $\alpha \to 0$, Dirichlet draws concentrate on sparse, highly uneven weight vectors, and the local influence approximation may become less accurate because the perturbation from uniform weights is no longer small. Conversely, as $\alpha \to \infty$, the weights collapse to the uniform vector and the procedure recovers the deterministic ERM solution with no bootstrap-induced uncertainty.

In practice, we tune $\alpha$ using a validation split to achieve empirical coverage or likelihood calibration:
\begin{itemize}
 \item \textbf{Regression.} For a given $\alpha$, Ribbon induces an epistemic predictive distribution via Dirichlet--weighted influence perturbations. The total predictive uncertainty is obtained by combining this epistemic component with an estimate of aleatoric variance. We tune $\alpha$ on a validation split so that the resulting predictive intervals achieve the target empirical coverage $(1-\tau)$.
  \item \textbf{Classification.} We select $\alpha$ using a validation criterion such as negative log-likelihood or Brier score, optionally breaking ties using expected calibration error (ECE).
\end{itemize}

Empirically, the Ribbon method is relatively insensitive to the precise choice of $\alpha$ once calibrated. However, including $\alpha$ as a tunable parameter provides a straightforward mechanism for aligning posterior--predictive coverage with empirical frequencies. The formal coverage result in Appendix~\ref{sec:theory} assumes that population coverage is continuous and monotone in $\alpha$ on the tuning range; this is a sufficient condition and may fail in finite samples or in highly nonmonotone models.

\section{Comparison Methods}
\label{sec:baselines}

\subsection{Implementation Details for Baseline Methods}
\label{sec:implementation}

This section summarizes the implementation details for the baseline uncertainty quantification methods used in our experiments. All methods are trained on the same datasets using identical network architectures and optimization settings unless otherwise noted. In all runtime tables below, reported times are \emph{post-hoc uncertainty-generation times only}. They exclude the shared base-model training time, data loading, and manual hyperparameter search. For methods that require additional fitted models after the base model, such as ensembles and bootstrap refits, the reported post-hoc time includes the additional training or refitting needed to generate the uncertainty draws.

\paragraph{Deep ensembles.}
For deep ensembles, we train independently initialized neural networks using the same architecture and training procedure. Each network is optimized by minimizing the mean squared error or cross-entropy loss, as appropriate for the task, using the Adam optimizer with a fixed learning rate and number of training steps. Epistemic predictive uncertainty is estimated from the empirical distribution of predictions across ensemble members. The post-hoc time reported for ensembles includes training the additional ensemble members beyond any shared reference fit and generating their predictions.

\paragraph{Monte Carlo dropout.}
For MC dropout, we train a neural network with dropout layers and retain dropout at test time. Predictive draws are generated by repeated stochastic forward passes through the trained model. The dropout probability is fixed across experiments or tuned on validation data, depending on the experiment. The post-hoc time reported for MC dropout includes only stochastic prediction generation after the trained dropout model is available.

\paragraph{Bayesian neural networks.}
Bayesian neural network baselines are implemented using Hamiltonian Monte Carlo with the No--U--Turn Sampler as provided by \texttt{NumPyro}. Independent Gaussian priors are placed on all network weights and biases, and a Gaussian likelihood with unknown noise variance is assumed for regression. Due to computational cost, this baseline is restricted to relatively small network architectures. The reported post-hoc time for this method includes posterior sampling and predictive generation.

\paragraph{Laplace approximation.}
The Laplace approximation is constructed around a deterministic maximum a posteriori estimate obtained by training a neural network to minimize the training loss with the specified regularization. The posterior over parameters is approximated by a Gaussian distribution whose covariance is given by a structured approximation to the Hessian or generalized Gauss--Newton matrix of the training loss. Prior precision, damping, likelihood scale, and curvature structure are tuned or specified consistently with the validation protocol for each experiment. Unless otherwise stated, references to Laplace in the empirical comparison denote the standard curvature-only Laplace approximation rather than a sandwich-corrected Laplace approximation. The reported post-hoc time for Laplace includes curvature construction, posterior sampling or linearized predictive propagation, and predictive summarization, but excludes the shared base-model fit.

\paragraph{Variational inference.}
Variational inference turns Bayesian sampling into an optimization problem involving the Kullback--Leibler divergence between an assumed posterior family and the true posterior. Because the true posterior is unknown, the objective is transformed into the evidence lower bound, which depends only on the observed data and the variational distribution. When reported, we use the \texttt{NumPyro} implementation of stochastic variational inference.

\paragraph{Bootstrap retraining.}
For bootstrap retraining, we generate $B$ bootstrap datasets by sampling the training data with replacement. A separate neural network is trained from scratch on each resampled dataset using the same architecture and optimization procedure as the base model. Although computationally expensive, bootstrap retraining serves as a robust nonparametric reference method. The reported post-hoc time includes training the bootstrap refits and generating predictions.

\paragraph{Bayesian bootstrap refitting.}
When computationally feasible, we also consider a direct Bayesian-bootstrap refitting baseline. For each draw, weights $w\sim\mathrm{Dirichlet}(1,\ldots,1)$ are sampled and the model is retrained or reoptimized using the weighted empirical loss. This baseline is the resampling target that Ribbon approximates through the influence-function linearization when $\alpha=1$.

\paragraph{Evaluation protocol.}
For all baseline methods, epistemic predictive distributions are summarized using empirical quantiles of the predictive draws. Predictive intervals and coverage statistics are computed pointwise over test inputs. Aleatoric uncertainty is incorporated post hoc using a shared noise variance estimate, as described in Section~\ref{sec:aleatoric}, to ensure a consistent comparison across methods.

\subsection{Estimating Aleatoric Uncertainty}
\label{sec:aleatoric}

All regression methods considered in our experiments produce a distribution over the conditional mean function $f(x;\theta)$ or its approximation, capturing epistemic uncertainty through variation in model parameters across draws or refits. To obtain predictive uncertainty for observations $y_\ast \mid x_\ast$, we additionally model aleatoric uncertainty via an observation noise model. Throughout the main regression experiments, we use a homoskedastic Gaussian likelihood,
\[
y \mid x,\theta,\sigma^2 \sim \mathcal{N}\!\big(f(x;\theta),\,\sigma^2\big),
\]
so that the predictive distribution marginalizes both parameter uncertainty and observation noise.
Given a collection of epistemic draws $\{f^{(b)}(x_\ast)\}_{b=1}^B$, we form predictive samples by adding independent noise
\[
y_\ast^{(b)}(x_\ast) = f^{(b)}(x_\ast) + \epsilon^{(b)}, \qquad
\epsilon^{(b)} \overset{\mathrm{iid}}{\sim} \mathcal{N}(0,\hat\sigma^2).
\]
We estimate the noise level from the training data using the same residual-variance estimator under the fitted mean predictor:
\[
\hat\sigma^2 = \frac{1}{n-1}\sum_{i=1}^n \left(y_i - f(x_i;\hat\theta)\right)^2.
\]
Thus the aleatoric component is estimated using the same protocol across methods, while the numerical value of $\hat\sigma^2$ may differ when the fitted mean predictors differ.

\section{Empirical Results}
\label{sec:experiments}

We now evaluate the empirical performance of Ribbon on regression and classification benchmarks. All experiments were conducted using a single base model fit per dataset for Ribbon, after which posterior--predictive draws were generated according to Section~\ref{sec:methods}. Comparisons are made to standard baselines for uncertainty quantification, including MC dropout, standard curvature-only Laplace approximations, deep ensembles, and, where computationally feasible, bootstrap refitting or HMC Bayesian neural networks.
All reported runtimes are post-hoc uncertainty-generation times only. For Ribbon and Laplace, this means the time after the shared base model has been trained. For ensembles, bootstrap retraining, and HMC Bayesian neural networks, this includes the additional fitting or sampling required to generate uncertainty draws. Runtime measurements exclude data loading and manual hyperparameter search.

\subsection{Synthetic Heteroskedastic Sine Regression}

We consider a one--dimensional regression problem with input--dependent Gaussian noise,
\[
y = \sin(x) + \sigma(x)\,\epsilon,
\qquad \epsilon \sim \mathcal{N}(0,1),
\]
where
\[
\sigma(x) = 0.05 + 0.30\left(\frac{|x|}{6}\right)^{1.2}.
\]
Training inputs are sampled uniformly as $x \sim \mathrm{Unif}[-\pi,\pi]$, which we treat as the in--distribution region. Out--of--distribution test points are drawn independently from $x \sim \mathrm{Unif}[4,6]$. Across $R=10$ independent trials, we generate fresh training datasets and evaluate uncertainty on independently sampled ID and OOD test sets.

Because the data-generating process is heteroskedastic while the fitted models use
comparatively simple homoskedastic uncertainty constructions, this experiment should be
viewed as a deliberately misspecified setting. Under such misspecification, Ribbon's
influence-function covariance can capture gradient variability that is not represented by
standard curvature-only Laplace approximations, helping explain its improved calibration
in this experiment.

All methods use the same neural network architecture: a fully connected network with a single hidden layer of width 10 and $\tanh$ activation. Models are trained using Adam to minimize mean squared error. We compare deep ensembles, bootstrap retraining, MC dropout, a standard curvature-only Laplace approximation around the MAP solution, a fully Bayesian neural network fit using \texttt{NumPyro}/NUTS, and the Ribbon pushforward method with a globally calibrated concentration parameter. Unless otherwise noted, Ribbon uses $B=100$ Dirichlet draws.

Uncertainty is assessed via epistemic coverage. For each method, we construct nominal 90\% bands from draws of the regression function $f(x)$ and compute empirical coverage with respect to the true function $\sin(x)$. This isolates epistemic uncertainty and excludes aleatoric noise from the evaluation metric. These latent-function intervals are not directly comparable to the predictive intervals reported for California Housing, which include an aleatoric residual component.

Table~\ref{tab:sine_results} reports the latent-function epistemic coverage results. Figure~\ref{fig:uq_epistemic} visualizes the corresponding epistemic uncertainty bands, while Figure~\ref{fig:uq_total} shows total predictive intervals after adding observation-noise uncertainty. This distinction is important: Table~\ref{tab:sine_results} evaluates uncertainty for the latent function $\sin(x)$, whereas Figure~\ref{fig:uq_total} visualizes uncertainty for noisy observations.

Several qualitative differences emerge. Deep ensembles substantially under-cover in both ID and OOD regions under this deliberately small architecture and training regime. This under-coverage indicates that independent initialization alone produced limited functional diversity for the latent mean function in this setting; it should not be interpreted as a general failure of ensemble methods. MC dropout achieves near-nominal coverage in distribution but fails on the OOD region, exhibiting zero coverage on $x \in [4,6]$. The standard curvature-only Laplace approximation under-covers in distribution and provides only moderate OOD expansion. Bootstrap retraining improves ID calibration relative to Laplace but still under-covers in the OOD region. Ribbon achieves near-nominal ID coverage while maintaining moderate OOD expansion, outperforming Laplace in ID calibration. The fully Bayesian neural network yields strong OOD coverage, though at substantially higher post-hoc computational cost.

\begin{table}[t]
\centering
\caption{Synthetic heteroskedastic sine regression results. We report empirical 90\% epistemic coverage of the true latent function $\sin(x)$ on in-distribution (ID) and out-of-distribution (OOD) test regions, along with average post-hoc uncertainty-generation wall-clock time in seconds. Values are mean $\pm$ standard deviation over $R=10$ independent trials. Times exclude the shared base-model fit, data loading, and manual hyperparameter search. For ensembles, bootstrap retraining, and NUTS, the post-hoc time includes the additional model fitting or posterior sampling needed to generate uncertainty draws.}
\label{tab:sine_results}
\begin{tabular}{lccc}
\toprule
\textbf{Method} & \textbf{ID Cov.} & \textbf{OOD Cov.} & \textbf{Time (s)} \\
\midrule
Deep Ensemble & $0.177 \pm 0.044$ & $0.373 \pm 0.146$ & $2.128 \pm 0.406$ \\
Bootstrap & $0.741 \pm 0.094$ & $0.605 \pm 0.217$ & $4.115 \pm 0.460$ \\
MC Dropout & $0.875 \pm 0.017$ & $0.000 \pm 0.000$ & $0.434 \pm 0.047$ \\
Laplace & $0.753 \pm 0.118$ & $0.745 \pm 0.106$ & $0.012 \pm 0.000$ \\
Ribbon & $0.902 \pm 0.071$ & $0.861 \pm 0.119$ & $0.033 \pm 0.000$ \\
NumPyro BNN (NUTS) & $0.823 \pm 0.079$ & $0.962 \pm 0.027$ & $19.068 \pm 1.423$ \\
\bottomrule
\end{tabular}
\end{table}

Overall, this experiment highlights the tradeoff between calibration and post-hoc uncertainty cost. Full posterior inference provides strong OOD uncertainty at high cost, while Ribbon achieves near-nominal in-distribution calibration with minimal additional computational overhead after the base model is trained. Local Gaussian approximations and ensemble methods exhibit varying degrees of under-coverage, particularly outside the training region.

\FloatBarrier

\begin{figure}[t]
\centering
\includegraphics[width=\linewidth]{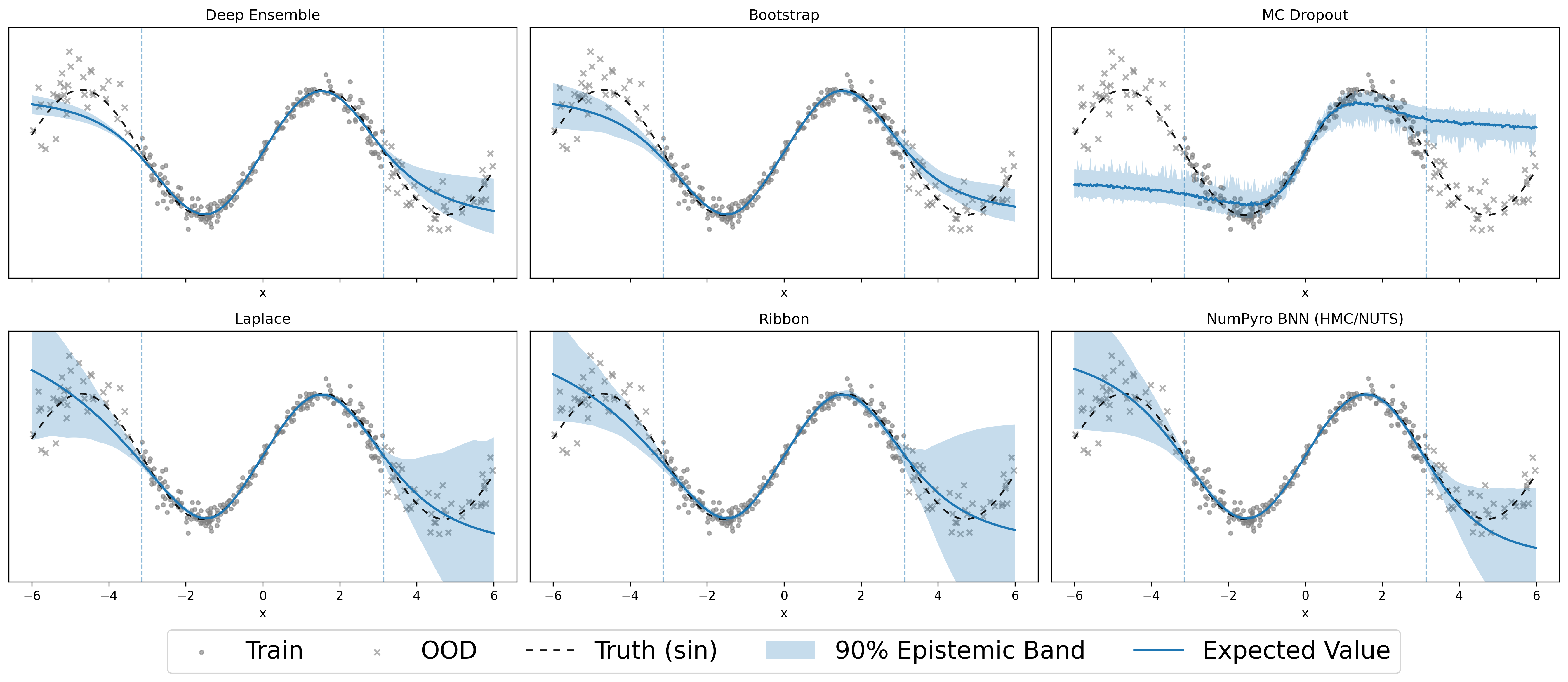}
\caption{Epistemic uncertainty on the heteroskedastic sine function. Each panel shows the predictive mean, 90\% epistemic interval, noisy data, and the true function. Gray shading marks the in-domain region $x\in[-\pi,\pi]$, and red shading marks the OOD region.}
\label{fig:uq_epistemic}
\end{figure}

\begin{figure}[t]
\centering
\includegraphics[width=\linewidth]{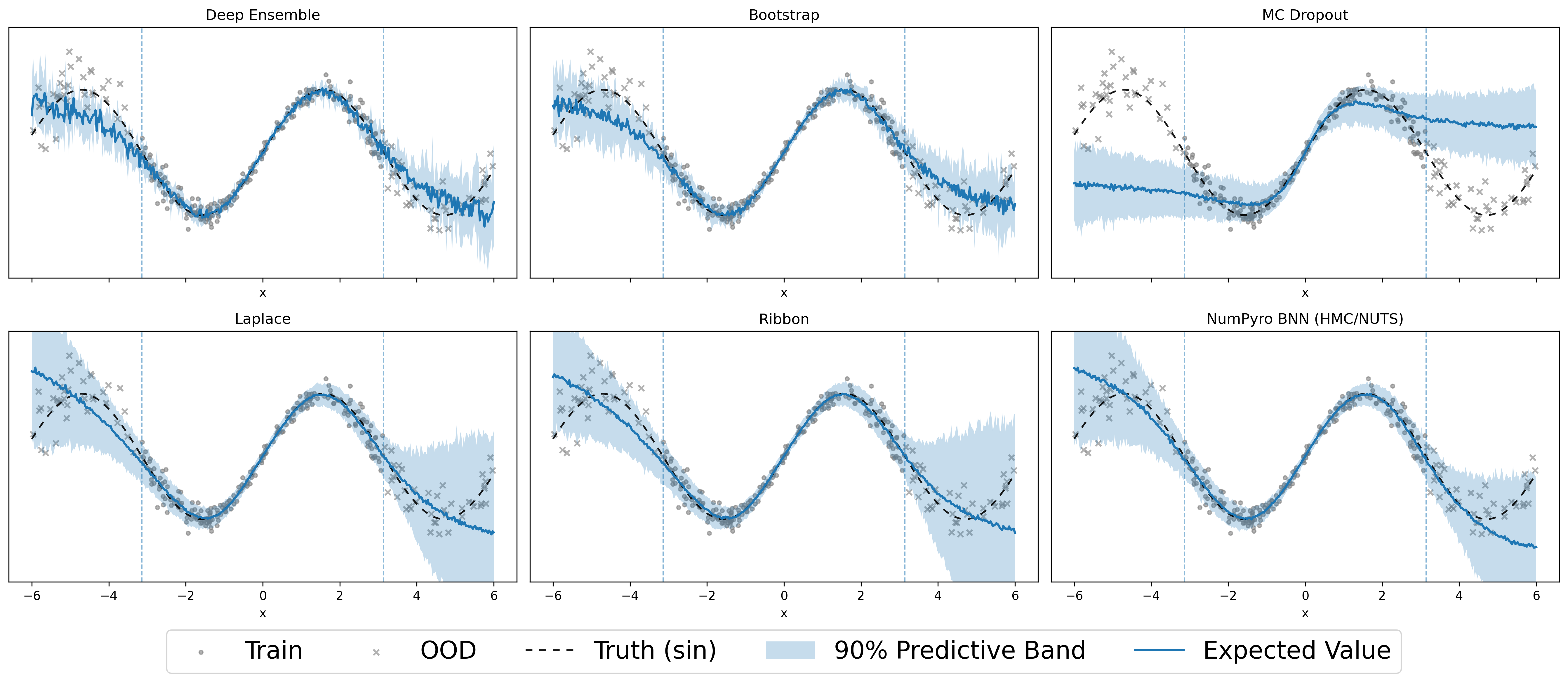}
\caption{Total predictive uncertainty on the sine dataset. Shaded regions represent 90\% predictive intervals obtained by combining epistemic variability with observation-noise uncertainty according to the method-specific predictive construction.}
\label{fig:uq_total}
\end{figure}

\FloatBarrier

\subsection{California Housing Regression}

We next evaluate predictive performance on the California Housing dataset. All input features are standardized using statistics from the training split, and a fixed train/validation/test partition is used across all methods. We compare Ribbon against deep ensembles, MC dropout, and the standard curvature-only Laplace approximation. For Ribbon, we use the Ribbon pushforward approximation defined in Section~\ref{sec:methods}, with validation tuning of the Dirichlet concentration parameter $\alpha$ to target 90\% predictive coverage, using $B=100$ Dirichlet draws.

The ID and OOD subsets are defined using distance from the training feature distribution after standardization. Specifically, test points in a pre-specified upper tail of Mahalanobis distance to the training feature mean are treated as OOD, while the remaining test points are treated as ID. This definition captures covariate shift relative to the training feature distribution, but it should not be interpreted as a guarantee of performance under arbitrary distribution shift. The continuous ranked probability score (CRPS) is computed from the empirical predictive distribution, with lower values indicating sharper and better-calibrated predictive distributions.

\begin{table}[t]
\centering
\caption{California Housing regression results. We report empirical coverage of nominal 90\% prediction intervals on in-distribution (ID), overall test (ALL), and out-of-distribution (OOD) subsets, along with distributional accuracy measured by CRPS and post-hoc uncertainty-generation time in seconds. Times exclude the shared base-model fit, data loading, and manual hyperparameter search.}
\label{tab:cahousing_cov_crps_time}
\begin{tabular}{lccccc}
\toprule
Method & ID Cov. (\%) & ALL Cov. (\%) & OOD Cov. (\%) & CRPS $\downarrow$ &  Time (s) \\
\midrule
Ribbon Pushforward & 91.8 & 91.4 & 89.6 & 0.3573 & 6.1 \\
Deep Ensemble & 91.7 & 91.4 & 90.0 & 0.3927 & 8.1 \\
MC Dropout    & 92.8 & 92.6 & 91.5 & 0.3593 & 9.7 \\
Laplace       & 92.0 & 95.8 & 99.2 & 0.4226 & 11.3 \\
\bottomrule
\end{tabular}
\end{table}

Table~\ref{tab:cahousing_cov_crps_time} summarizes empirical coverage of nominal 90\% prediction intervals, distributional accuracy measured by CRPS, and post-hoc uncertainty-generation time. Ribbon achieves near-nominal coverage on the test set, with competitive OOD performance and the lowest CRPS among the methods considered. MC dropout performs similarly in CRPS and slightly over-covers, while deep ensembles achieve similar coverage but somewhat worse CRPS.

The comparison with deep ensembles should be interpreted carefully. Because the ensemble averages multiple independently trained networks, its performance can reflect both uncertainty quantification and improved optimization or model averaging relative to a single fitted network. Thus the California Housing experiment does not fully isolate the effect of the post-hoc UQ method from the quality of the fitted mean function. The cleaner comparison among single-fit uncertainty methods is Ribbon versus MC dropout and Laplace-style local approximations around a single trained model.

In contrast, the standard curvature-only Laplace approximation performs poorly on this benchmark under the reported implementation, exhibiting over-coverage and the worst CRPS. This result should be interpreted in light of the specific curvature approximation, prior precision, damping, and likelihood-scale choices used in the experiment. The strong over-coverage suggests that this local Gaussian posterior approximation is too diffuse under the reported implementation choices.

For the California Housing experiments, all single-fit methods use the same base predictive model: a fully connected feedforward neural network with a single hidden layer of width 64 and $\tanh$ activation. The total number of trainable parameters is 641.

\subsection{MNIST Classification}

We conclude with the MNIST handwritten--digit classification benchmark. We train a deliberately small convolutional neural network under a constrained training budget to emphasize calibration behavior rather than peak classification accuracy. The network consists of two convolutional layers followed by a single fully connected classification layer and is evaluated on the held--out test set. In this experiment, Ribbon and Laplace use the same full-parameter damped PSD--GGN curvature approximation around the shared base model. For Ribbon, the Dirichlet concentration parameter $\alpha$ is selected on a validation set using Brier score, ECE, or NLL, giving three reported calibrated variants. We compare these local full-parameter methods against MC dropout and a deep ensemble.

\begin{table}[t]
\centering
\caption{MNIST-10 classification results for the CNN full-parameter PSD--GGN experiment. We report mean $\pm$ standard deviation of test-set accuracy, Brier score, expected calibration error (ECE), and negative log-likelihood (NLL) over repeated trials. Lower values indicate better performance for Brier, ECE, and NLL.}
\label{tab:mnist_results}
\begin{tabular}{lcccc}
\toprule
Method & Accuracy (\%) $\uparrow$ & Brier $\downarrow$ & ECE $\downarrow$ & NLL $\downarrow$ \\
\midrule
MC Dropout       & $85.41 \pm 1.39$ & $0.2417 \pm 0.0179$ & $0.1302 \pm 0.0098$ & $0.5380 \pm 0.0339$ \\
Deep Ensemble    & $87.66 \pm 0.69$ & $0.2028 \pm 0.0091$ & $0.0976 \pm 0.0146$ & $0.4582 \pm 0.0196$ \\
Laplace Full     & $85.85 \pm 1.25$ & $0.2201 \pm 0.0171$ & $0.0791 \pm 0.0067$ & $0.4863 \pm 0.0310$ \\
Ribbon Full (Brier) & $85.85 \pm 1.26$ & $0.2204 \pm 0.0172$ & $0.0799 \pm 0.0066$ & $0.4866 \pm 0.0311$ \\
Ribbon Full (ECE)   & $85.84 \pm 1.26$ & $0.2204 \pm 0.0172$ & $0.0797 \pm 0.0065$ & $0.4867 \pm 0.0313$ \\
Ribbon Full (NLL)   & $85.82 \pm 1.28$ & $0.2204 \pm 0.0172$ & $0.0795 \pm 0.0067$ & $0.4868 \pm 0.0312$ \\
\bottomrule
\end{tabular}
\end{table}

Table~\ref{tab:mnist_results} summarizes test--set performance. The deep ensemble achieves the highest accuracy and the best Brier score and NLL, reflecting the benefit of independently trained models even under the constrained training budget. The full-parameter Laplace and Ribbon variants are substantially better calibrated than MC dropout by ECE and also improve Brier score and NLL relative to MC dropout. The Ribbon variants are nearly indistinguishable from full-parameter Laplace on this benchmark, with small differences across the Brier-, ECE-, and NLL-selected concentration values. This behavior is consistent with the theoretical connection in Section~\ref{sec:laplace}: in relatively benign classification settings with a shared local curvature approximation and global scale tuning, Ribbon can closely track curvature-based Laplace rather than dominate it. The main empirical conclusion from this MNIST experiment is therefore that Ribbon is competitive and stable for full-parameter local uncertainty quantification, while the ensemble remains the strongest method on accuracy and proper scoring rules in this particular setting.

\section{Discussion and Future Directions}
\label{sec:discussion}

This work introduces Ribbon, a scalable method for predictive uncertainty quantification based on influence--function linearization and Dirichlet reweighting. Conceptually, Ribbon combines a Bayesian-bootstrap-like source of uncertainty with a frequentist calibration step. The Bayesian-bootstrap-like component comes from approximating Dirichlet-reweighted refitting through a gradient-based influence-function linearization, while frequentist calibration is introduced by tuning the Dirichlet concentration parameter $\alpha$ on validation data.

For $\alpha=1$, Ribbon approximates the standard Dirichlet-reweighted Bayesian-bootstrap refitting target. For general tuned $\alpha$, it becomes a calibrated Dirichlet-reweighting generalization whose global covariance scale is adjusted by validation.

The theoretical analysis in Appendix~\ref{sec:theory} demonstrates that, in regular finite-dimensional M-estimation problems, Ribbon achieves first--order equivalence to both the flat-prior Laplace approximation and the Bayesian bootstrap under appropriate conditions. Under correct likelihood specification and $\alpha=1$, the asymptotic covariance of Ribbon coincides with the Laplace covariance. Under misspecification, Ribbon generalizes naturally to the robust sandwich form $H^{-1}H_FH^{-1}/n$, capturing additional sampling variability due to heteroskedasticity, outliers, or other departures from the assumed likelihood.

The Dirichlet concentration parameter $\alpha$ plays a transparent role in controlling dispersion. When tuned via validation, $\alpha$ can restore nominal coverage in settings where the sufficient monotonicity assumptions hold approximately. This validation step is not a finite-sample coverage guarantee; rather, it is a consistency result under continuity, monotonicity, and uniform convergence assumptions.

From a computational standpoint, Ribbon requires only a single model fit and a curvature operator derived from the training loss. Curvature can be estimated via generalized Gauss--Newton, KFAC, empirical Fisher, or low--rank approximations, all of which are compatible with modern automatic-differentiation frameworks. The method's post-hoc cost is dominated by matrix--vector solves and Jacobian--vector products, scaling linearly with the number of bootstrap draws $B$.

The empirical studies show that Ribbon often improves calibration relative to standard curvature-only Laplace approximations and remains competitive with dropout and ensembles. In small, well-specified problems, Laplace may remain competitive; however, in settings characterized by mild misspecification or heteroskedasticity, Ribbon can provide more accurate predictive intervals without the need for multiple model fits.

Despite these advantages, Ribbon inherits the local nature of influence--function approximations. Its validity depends on smoothness of the loss landscape and well-conditioned curvature. In highly nonconvex, overparameterized, or discontinuous regimes, the linearization may fail to capture nonlinear parameter interactions, leading to under- or over-estimated uncertainty. Damping, structured curvature approximations, and optional clipping of influence updates can mitigate but not eliminate this limitation. The formal asymptotic results do not by themselves justify arbitrary overparameterized neural-network settings; the neural-network experiments should be interpreted as empirical evidence for the practical usefulness of the approximation.

Furthermore, tuning $\alpha$ only rescales global dispersion. It cannot correct arbitrary directional misspecification in $H^{-1}H_FH^{-1}$, nor can it guarantee reliable uncertainty under arbitrary distribution shift. For classification, nonlinear transformations such as the softmax can also produce skewed or heavy-tailed predictive probability distributions, suggesting that post-hoc calibration may remain useful.

Several extensions merit further investigation. Higher--order influence expansions could enlarge the region of validity of the linearization. Adaptive curvature learning could improve the fidelity of the $H^{-1}$ approximation at scale. Integrating Ribbon with conformal prediction offers a route to finite-sample coverage guarantees, complementing the asymptotic calibration result. Finally, extending the framework to structured prediction, sequence modeling, and test-time adaptation represents a promising avenue for future research.

\bibliographystyle{plainnat}
\bibliography{references}

\clearpage
\appendix
\section{Theoretical Results}
\label{sec:theory}

Throughout, we work under the following regularity conditions. The assumptions are classical local M-estimation assumptions and are strongest when the parameter dimension is fixed or grows slowly relative to sample size. They do not by themselves justify arbitrary overparameterized neural-network settings.

\begin{assumption}\label{assump:regularity}
Let $\ell(z_i,\theta)$ be the per-sample loss and let $\hat\theta$ be the unique minimizer of $L_n(\theta)=n^{-1}\sum_{i=1}^n\ell(z_i,\theta)$. Assume:
\begin{enumerate}
\item[(A1)] \textbf{Smoothness and identifiability.} In a neighborhood $\mathcal{N}$ of $\hat\theta$, the map $\theta\mapsto \nabla_\theta \ell(z_i,\theta)$ is twice continuously differentiable for almost every $z_i$, and $\nabla_\theta L_n(\hat\theta)=0$ with $\hat\theta$ unique.
\item[(A2)] \textbf{Well-conditioned curvature.} The mean Hessian $H=n^{-1}\sum_{i=1}^n\nabla^2_\theta\ell(z_i,\hat\theta)$ is invertible and $\|H^{-1}\|=O_p(1)$. Moreover, the Hessian is locally Lipschitz in $\theta$.
\item[(A3)] \textbf{Gradient moments and empirical stability.} Writing $g_i=\nabla_\theta\ell(z_i,\hat\theta)$ and $H_i=\nabla^2_\theta\ell(z_i,\hat\theta)$, we have $n^{-1}\sum_{i=1}^n\|g_i\|^2=O_p(1)$ and $H_F=n^{-1}\sum_{i=1}^n g_i g_i^\top=O_p(1)$.
\item[(A4)] \textbf{Weighted optimum exists and is close.} For bootstrap weights $w$ with $\tilde{w}=nw-j$ and fixed $\alpha\in[\alpha_{\min},\alpha_{\max}]$ with $0<\alpha_{\min}<\alpha_{\max}<\infty$, the weighted minimizer $\hat\theta_w$ exists in $\mathcal{N}$ and $\|\hat\theta_w-\hat\theta\|=O_p(n^{-1/2})$.
\item[(A5)] \textbf{Weighted Hessian stability.} For the Dirichlet weights considered,
\[
\left\|\sum_{i=1}^n w_i H_i - \frac1n\sum_{i=1}^n H_i\right\|=O_p(n^{-1/2}).
\]
\end{enumerate}
\end{assumption}

\begin{theorem}[Asymptotic Covariance of Ribbon]
\label{thm:asymp-cov}
Suppose Assumption~\ref{assump:regularity} holds. Let $g_i = \nabla_\theta \ell(z_i,\hat\theta)$, $H = n^{-1}\sum_i \nabla_\theta^2 \ell(z_i,\hat\theta)$, and $H_F = n^{-1}\sum_i g_i g_i^{\!\top}$.
Under the $\mathrm{Dirichlet}(\alpha, \ldots, \alpha)$ weighting scheme and the influence update
\[
\Delta\theta_w = -H^{-1}\!\left(\frac{1}{n}G^\top \tilde{w}\right),\qquad G=[g_1,\dots,g_n]^\top,\; \tilde{w}=nw-j,
\]
with fixed $\alpha>0$ and $G^\top j=0$, we have, conditional on the data,
\[
  \mathrm{Var}(\Delta\theta_w)
  = \frac{1}{n\alpha+1}\,H^{-1} H_F H^{-1}.
\]
For $\alpha=1$,
\[
  \mathrm{Var}(\Delta\theta_w)
  = \frac{1}{n}\,H^{-1} H_F H^{-1} + o_p(n^{-1}).
\]
If the loss is a correctly specified negative log likelihood so that $H_F\to_p H$, then for $\alpha=1$,
$\mathrm{Var}(\Delta\theta_w) = n^{-1}H^{-1} + o_p(n^{-1})$.
For penalized objectives or nonzero empirical mean gradients, the same result holds with $H$ replaced by the penalized curvature and $H_F$ replaced by the centered gradient second moment.
\end{theorem}

\begin{proof}
Condition on the data $\{z_i\}_{i=1}^n$ so that $G$ and $H$ are deterministic.
For $w\sim\mathrm{Dirichlet}(\alpha,\ldots,\alpha)$,
\[
\mathrm{Cov}(\tilde w)
=
\frac{n}{n\alpha+1}\left(I-\frac1njj^\top\right).
\]
Therefore,
\[
\mathrm{Var}\!\left(\frac{1}{n}G^\top \tilde{w}\right)
= \frac{1}{n^2} G^\top \mathrm{Cov}(\tilde{w}) G
= \frac{1}{n(n\alpha+1)} G^\top\!\left(I-\frac{1}{n}jj^\top\right)G.
\]
If $G^\top j=0$, the rank-one correction vanishes and
\[
\mathrm{Var}\!\left(\frac{1}{n}G^\top \tilde{w}\right)
= \frac{1}{n(n\alpha+1)}G^\top G
= \frac{1}{n\alpha+1}H_F.
\]
Propagating through the linear map $-H^{-1}$ gives the stated covariance. For $\alpha=1$, $(n+1)^{-1}=n^{-1}+O(n^{-2})$. Under correct likelihood specification, the information identity and empirical convergence imply $H_F\to_p H$.
If $G^\top j\ne0$, the same proof gives the centered matrix $H_{F,c}=n^{-1}G^\top(I-n^{-1}jj^\top)G$.
\end{proof}

\begin{corollary}[Epistemic Predictive Covariance]
\label{cor:predictive}
Let $J_{x_\ast} = \partial f(x_\ast;\hat\theta)/\partial\theta^\top$. Then
\[
  \mathrm{Var}\!\big[f(x_\ast)\big]
  = J_{x_\ast}\mathrm{Var}(\Delta\theta_w)J_{x_\ast}^\top
  =
  \frac{1}{n\alpha+1}J_{x_\ast}H^{-1}H_FH^{-1}J_{x_\ast}^\top
\]
in the unregularized exact-ERM case, with the corresponding centered and penalized replacements otherwise.
\end{corollary}

\begin{proof}
Immediate from Theorem~\ref{thm:asymp-cov} and linearity of variance under the deterministic linear map $v\mapsto J_{x_\ast}v$.
\end{proof}

\begin{theorem}[Local Accuracy of the Ribbon Update]
\label{thm:update-accuracy}
Under Assumption~\ref{assump:regularity}, let
\[
\hat\theta_w = \arg\min_\theta \sum_{i=1}^n w_i \ell(z_i,\theta),
\qquad
\tilde\theta_w = \hat\theta - H^{-1}\!\left(\frac{1}{n}G^\top \tilde{w}\right),
\quad \tilde{w} = n w - j.
\]
Then, conditional on the data,
\[
\|\hat\theta_w - \tilde\theta_w\| = O_p(n^{-1}).
\]
\end{theorem}

\begin{proof}
Define
\[
\Delta^* = \hat\theta_w - \hat\theta,
\qquad
\Delta_{\mathrm{IF}} = \tilde\theta_w - \hat\theta = -H^{-1}U_n,
\quad
U_n := \sum_{i=1}^n w_i g_i = \frac{1}{n}G^\top \tilde{w},
\]
where the final equality uses $G^\top j=0$.
The weighted estimator satisfies
\[
0 = \sum_{i=1}^n w_i \nabla_\theta \ell(z_i,\hat\theta_w).
\]
By a first-order Taylor expansion around $\hat\theta$ with remainder,
\[
\nabla_\theta \ell(z_i,\hat\theta_w)
= g_i + H_i \Delta^* + R_i,
\qquad
\|R_i\| \le \frac12 L_H \|\Delta^*\|^2.
\]
Substituting and rearranging yields
\[
H\Delta^*
=
- U_n
- E_H \Delta^*
- \sum_{i=1}^n w_i R_i,
\]
where
\[
H := \frac{1}{n}\sum_{i=1}^n H_i,
\qquad
E_H := \sum_{i=1}^n w_i H_i - H.
\]
Multiplying by $H^{-1}$ and subtracting $\Delta_{\mathrm{IF}} = -H^{-1}U_n$ gives
\[
\Delta^* - \Delta_{\mathrm{IF}}
=
- H^{-1} E_H \Delta^*
- H^{-1} \sum_{i=1}^n w_i R_i.
\]
By Assumption~\ref{assump:regularity}(A4), $\|\Delta^*\|=O_p(n^{-1/2})$. By Assumption~\ref{assump:regularity}(A5), $\|E_H\|=O_p(n^{-1/2})$. Thus the first term is $O_p(n^{-1})$. Since $\sum_i w_i=1$ and $\|R_i\|\le \frac12 L_H\|\Delta^*\|^2$, the second term is also $O_p(n^{-1})$. Combining the bounds gives the result.
\end{proof}

\begin{remark}
Assumption~\ref{assump:regularity}(A4) is a local stability condition. Theorem~\ref{thm:update-accuracy} does not assert global convergence of the weighted ERM; rather, it gives the local error of the influence approximation conditional on the weighted solution remaining in the regular neighborhood of $\hat\theta$.
\end{remark}

For validation-based calibration, we assume:
\begin{enumerate}[label=(B\arabic*)]
\item \label{assump:B1}
$C(\alpha)$ is continuous and strictly monotone in $\alpha$ on a compact set
$\mathcal A \subset (0,\infty)$, with $C(\alpha_\tau)=1-\tau$ for some
$\alpha_\tau\in\mathcal A$.
\item \label{assump:B2}
A uniform law of large numbers holds:
\[
\sup_{\alpha\in\mathcal A} |\widehat C_m(\alpha)-C(\alpha)| \to_p 0 \quad \text{as } m\to\infty.
\]
\end{enumerate}

\begin{theorem}[Predictive Coverage via Tuned $\alpha$]
\label{thm:coverage}
Let $1-\tau\in(0,1)$ denote a nominal predictive coverage level.
For each $\alpha>0$, let
\[
I_{1-\tau}(x;\alpha) = [q_{\tau/2}(x;\alpha),\, q_{1-\tau/2}(x;\alpha)]
\]
denote the equal-tailed $(1-\tau)$ predictive interval of the Ribbon procedure at covariate $x$.
Define the population and validation coverages
\[
  C(\alpha) = \Pr\{Y\in I_{1-\tau}(X;\alpha)\},
  \qquad
  \widehat{C}_m(\alpha) = \frac{1}{m}\sum_{k=1}^m
     \{Y_k \in I_{1-\tau}(X_k;\alpha)\}.
\]
Let $\hat\alpha=\arg\min_{\alpha\in\mathcal{A}}|\widehat{C}_m(\alpha)-(1-\tau)|$.
Then $\hat\alpha\to_p\alpha_\tau$ and
\[
  \Pr\{Y\in I_{1-\tau}(X;\hat\alpha)\}\to_p 1-\tau.
\]
\end{theorem}

\begin{proof}
By (B2), $\widehat{C}_m(\alpha)\to_p C(\alpha)$ uniformly on $\mathcal{A}$.
Let $T(\alpha)=|C(\alpha)-(1-\tau)|$ and $\widehat{T}_m(\alpha)=|\widehat{C}_m(\alpha)-(1-\tau)|$.
Uniform convergence implies $\sup_{\alpha\in\mathcal{A}}|\widehat{T}_m(\alpha)-T(\alpha)|\to_p 0$.
By (B1), $T(\alpha)$ has a unique minimizer at $\alpha_\tau$.
Argmin consistency under uniform convergence then yields $\hat\alpha\to_p\alpha_\tau$.
By continuity of $C(\cdot)$, $C(\hat\alpha)\to_p C(\alpha_\tau)=1-\tau$.
\end{proof}

\begin{remark}
Theorem~\ref{thm:coverage} formalizes the validation-tuning step under sufficient monotonicity assumptions. It is not a finite-sample coverage guarantee and may fail when coverage is nonmonotone in $\alpha$ or when the validation sample is too small to estimate coverage reliably.
\end{remark}

\begin{corollary}[First--Order Covariance Equivalence and Calibration]
\label{cor:summary}
Under Assumption~\ref{assump:regularity}, the covariance of the scaled Ribbon perturbation satisfies
\[
  \mathrm{Var}\!\left(\sqrt{n}\,(\tilde\theta_w-\hat\theta)\right)
  \to
  \alpha^{-1}H^{-1}H_FH^{-1}
\]
for fixed $\alpha>0$, with the standard Bayesian-bootstrap case $\alpha=1$ yielding $H^{-1}H_FH^{-1}$.
Under an additional conditional central limit theorem for the Dirichlet-weighted score perturbation, the corresponding Gaussian limit follows.
If (B1)--(B2) also hold, the predictive distribution tuned by $\hat\alpha$ achieves asymptotically correct coverage.
\end{corollary}

\begin{proof}
The covariance statement follows from Theorem~\ref{thm:asymp-cov} and the fact that $n/(n\alpha+1)\to 1/\alpha$. The optional Gaussian limit requires the stated additional conditional central limit theorem. The coverage statement is Theorem~\ref{thm:coverage}.
\end{proof}

\end{document}